\documentclass[conference]{IEEEtran}
\IEEEoverridecommandlockouts
\usepackage{cite}
\usepackage{amsmath,amssymb,amsfonts}
\usepackage{algorithmic}
\usepackage{graphicx}
\usepackage{textcomp}
\usepackage{xcolor}
\usepackage{microtype}
\usepackage{hyperref}
\usepackage{url}
\usepackage{booktabs}
\usepackage{graphicx}
\usepackage{amsmath,amsthm,amsfonts}

\newtheorem{thm}{Theorem}
\newtheorem{lem}[thm]{Lemma}
\usepackage{multirow}
\usepackage[numbers]{natbib}
\usepackage{subfigure}
\begin{document}

\title{ Zero-Shot Link Prediction in Knowledge Graphs with Large Language Models\\

}
\author{Mingchen Li\textsuperscript{\normalfont 1}, Chen Ling\textsuperscript{\normalfont 2}, Rui Zhang\textsuperscript{\normalfont 1*}\thanks{* Corresponding author}, Liang Zhao\textsuperscript{\normalfont 2*}\thanks{* Corresponding author} \\ \textsuperscript{1}Division of Computational Health Sciences, Department of Surgery, University of Minnesota Twin Cities, USA \\ \textsuperscript{2}Department of Computer Science, Emory University, USA\\ \textsuperscript{1}\{li003378,zhan1386\}@umn.edu,  \textsuperscript{2}\{chen.ling, liang.zhao\}@emory.edu}


\maketitle

\begin{abstract}
Zero-shot link prediction (ZSLP) on knowledge graphs aims at automatically identifying relations between given entities. Existing methods primarily employ auxiliary information to predict tail entity given head entity and its relation, yet face challenges due to the occasional unavailability of such detailed information and the inherent simplicity of predicting tail entities based on semantic similarities. Even though Large Language Models (LLMs) offer a promising solution to predict unobserved relations between the head and tail entity in a zero-shot manner, their performance is still restricted due to the inability to leverage all the (exponentially many) paths' information between two entities, which are critical in collectively indicating their relation types. To address this, in this work, we introduce a \textbf{C}ondensed \textbf{T}ransition Graph Framework for Zero-Shot \textbf{L}ink \textbf{P}rediction (CTLP), which encodes all the paths' information in linear time complexity to predict unseen relations between entities, attaining both efficiency and information preservation. Specifically, we design a condensed transition graph encoder with theoretical guarantees on its coverage, expressiveness, and efficiency. It is learned by a transition graph contrastive learning strategy. Subsequently, we design a soft instruction tuning to learn and map the all-path embedding to the input of LLMs.
Experimental results show that our proposed CTLP method achieves state-of-the-art performance on three standard ZSLP datasets.\footnote{The code is available here: \url{https://github.com/ToneLi/Graph_LLM_link_predcition}}

\end{abstract}


\section{Introduction}
Knowledge graphs (KGs)~\citep{li2020multi}, which are rich structured representations of real-world entities and their relations, are foundational in various applications, from semantic search to recommendation systems \citep{li2022semantic,li2023understand,li2023petailor}. However, despite their extensive utility, most KGs suffer from an intrinsic shortcoming: incompleteness, which makes them unable to encapsulate the full breadth of evolving concepts. The task of KG link prediction therefore arises that aims to predict relation types between given entities in KGs \citep{wang2021survey}. Conventional methods \citep{bordes2013translating,kazemi2018simple,rossi2021knowledge} in KG link prediction learn low-dimensional representations of entities and 
relations, which are then used to infer links between entities. However, these methods rely on observed links to infer missing ones, limited by  \textit{1) the data they have been trained on} and \textit{2) struggling to generalize to unseen entities or relations}. In light of both limitations, the motivation for Zero-Shot Link Prediction on KGs becomes apparent.




    \begin{figure}
        \centering
        \includegraphics[width=0.5\columnwidth]{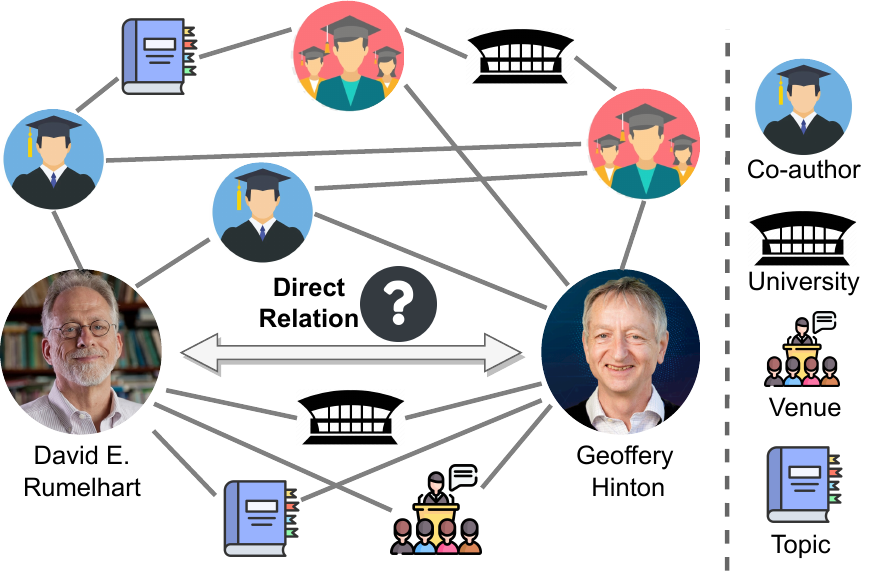}  
        \vspace{-1mm}
        \caption{Example of predicting the direct relation between two scholar entities in zero-shot by only giving their local neighboring information. Local neighboring information indicates there are many multi-hop paths with different intermediate entities, including co-author entity, venue entity, topic entity, etc.}
        \label{fig: intro}
        \vspace{-4mm}
    \end{figure}

To date, zero-shot link prediction on KGs is still under-explored, where previous works \citep{qin2020generative,geng2021ontozsl,li2022hierarchical} have utilized textual features or ontologies as auxiliary information to discern relationships between head entities, seen relations, and tail entities. Specifically, they have focused on predicting tail entities given a head entity and a relation, regardless of whether that relation has been previously encountered in the training data. However, rich text features and ontologies may not always be available in KG. In addition, when given head entities and the relation, predicting tail entities is relatively easy since semantic similarity would greatly reduce the pool of candidate tail entities \citep{zeng2020degree}.

In this work, as illustrated in Figure \ref{fig: intro}, we focus on an exploration of a novel zero-shot link prediction task on KGs, which aims to predict unseen relation types between head and tail entities without depending on learned model parameters that tie seen relations to specific head and tail entities. As a general task solver, Large Language Models (LLMs) can intuitively be employed to predict the unseen relation \citep{ling2023domain,bai2024beyond}. Existing methods \citep{pan2024unifying,ye2023natural} suggest querying LLMs about the relationship between specific head and tail entities, often enriching these queries with information about neighbors. However, general neighbors are insufficient and distracting for link prediction. Instead, it is important to focus on the (higher-order) neighbors that form any paths between head and tail entities. As shown in Figure \ref{fig: intro}, to predict the direct relationship between two scholar entities: \textit{David E. Rumelhart} and \textit{Geoffrey E. Hinton}, there exist many multi-hop paths connecting both entities, and a high portion of these paths are related to co-authorship. This path information can provide a more concentrated source of relational cues than feeding the local neighboring information as a whole. LLMs can thus leverage path information as in-context demonstrations to more precise predictions, i.e., co-authorship. However, extracting and summarizing these paths is fraught with computational hurdles, given the NP-hard nature of path retrieval between entities in million-scale KGs~\citep{michel2023path,ling2023open}.

To balance the trade-off between efficiency and information preservation of information encoding between head and tail entities toward link prediction, we propose a novel framework: \textbf{C}ondensed \textbf{T}ransition Graph Framework for Zero-Shot \textbf{L}ink \textbf{P}rediction (CTLP), which distill all the paths' information into a \textit{condensed transition graph} embedding with linear time complexity. Then, it allows LLMs to fully incorporate path information to predict relation types given a head entity and tail entity. Specifically, CTLP consists of three main components: (1) A new condensed transition graph encoder is proposed to predict the embedding aggregated from all the condensed paths between head and tail entities via different edges, with coverage, expressiveness, and efficiency guarantees; (2) A graph-centric contrastive learning is designed to learn the condensed transition graph encoder; (3) to encode the condensed graph information into the LLMs, we propose a new prefix-tuning method to encode graph information into the task instructions. We provide our key contributions as follows: 

\begin{itemize} 
    \item \textbf{Problem}. We propose to solve the task of zero-shot link prediction in KGs by encoding all the paths' information in linear time complexity and letting LLMs make zero-shot predictions with soft prompt tuning. 
    \item \textbf{Technique}. We propose to summarize all paths between given entities into a condensed transition graph while retaining the same expressiveness with a theoretical guarantee. Graph-centric contrastive learning is designed to ensure the path information is being injected into the soft prompt.
    \item \textbf{Experiments}. We conduct a thorough analysis of our method both quantitatively and qualitatively, which demonstrates better performance than existing state-of-the-arts. Additional experiments, including the ablation study and sensitivity analysis, illustrate the robustness of CTLP.
\end{itemize}
\section{Related Work}

\noindent\textbf{Link Prediction.} Many studies~\citep{yang2014embedding, bordes2013translating,balavzevic2019tucker,li2022hierarchical,zhang2023making,yao2023exploring,ye2023natural} have been proposed to better predict the relationship between the head and tail entity. The translation-based model TransE~\citep{bordes2013translating} requires that the tail entity embedding is close to the sum of the head and relation embeddings; 
The non-bilinear model TuckER~\citep{balavzevic2019tucker} utilizes the tucker decomposition to build the connection between different knowledge graph triples. KG-BERT~\citep{yao2019kg} regards entities and relations as textual sequences, transforming the link prediction task into a classification problem. KopA~\citep{zhang2023making} feeds the entity and relation embedding into LLM  in the format of soft-prompting.  Although performance has been achieved incrementally, these approaches in their original form are unable to learn embeddings for unseen relations. This is because they learn entities and relation embeddings using the topological structure of the KG.

\noindent\textbf{Zero-shot Link Prediction.} Previous works tend to predict the tail by giving the head entity and unseen relation~\citep{qin2020generative}. \cite{qin2020generative} used textual information of the relation as auxiliary information and applied a Zero-Shot Generative Adversarial Network (ZSGAN) to learn the unseen relation embedding for the task. A hierarchical n-gram framework (HNZSLP)~\citep{li2022hierarchical} is proposed to use the structural information in relational n-grams for zero-shot link prediction.  Despite the success, there exists a research gap in investigating the selection of relations from a given candidate relation set without training on the seen relation dataset.
To achieve this objective, an intuitive idea is to utilize the LLMs, such as LLama~\citep{touvron2023llama}, to predict the relations,  for example,  by constructing the prompt and feeding the head entity and tail entity in the format of \textit{what is the relationship between the head entity and tail entity?} to the LLM. Unfortunately, this method is easy to cause hallucinations, because it fails to use enough knowledge to guide the LLM. So, in this work, we try to explore how to encode the path information between the head entity and the tail entity to predict the relationship in a new zero-shot learning setting.


\section{Problem Statement}
\label{con:task defination}
This section formulates our problem of \emph{Zero-shot Link Prediction}.

\noindent\textbf{Zero-shot link prediction on (knowledge) graphs.} This task is formulated as predicting the relation type between any two entities of a knowledge graph without any training on the current knowledge graph. More specifically, we aim at predicting the relation type $r\in C_{(s,t)}$ between a head entity $s$ and a tail entity $t$, where $C_{(s,t)}$ is the set of possible relation types.

The relation between $s$ and $t$ is well-centered among all the possible routes through which $s$ could correlate to $t$, which can be collectively formulated as an \emph{$(s,t)$-transition graph}.

 \paragraph{Transition graph.} 
For any pair of two entities $(s,t)$ in a knowledge graph, all the paths from head entity $s$ to tail entity $t$ collectively form an $(s,t)$-transition graph. Figure \ref{fig: intro} exemplifies an $(s,t)$-transition graph where $s$ is ``David E. Rumelhart'' and $t$ is `` Geoffery Hinton''. In practice, the length of paths can be upper-bounded by an integer $k$, which can usually be set as the diameter of the knowledge graph.
 To be specific, let $\mathcal{G}=(V,E)$ denote an $(s,t)$-transition graph consisting of a set of node $V$ and a set of edges $E\in V \times V$. We denote by $n$ the number of nodes in $\mathcal{G}$ and by $m$ its number of edges. 
A path from head entity $s$ and tail entity $v$ is denoted by $\pi=[v_0,r_0, v_1,...,v_{j-1}, r_i, v_j]$ such that $v_0=s, v_j=t$, $(v_j,v_{j+1}) \in E$. The textual description of the path $\pi$ is described as: $T(\pi)=$ \textit{the relationship between $v_0$ and $v_1$ is $r_0$ ,..., the relationship between $v_{j-1}$ and $v_j$ is $r_i$}.

\begin{figure*}[t]
        \centering
        \includegraphics[width=0.8\textwidth]{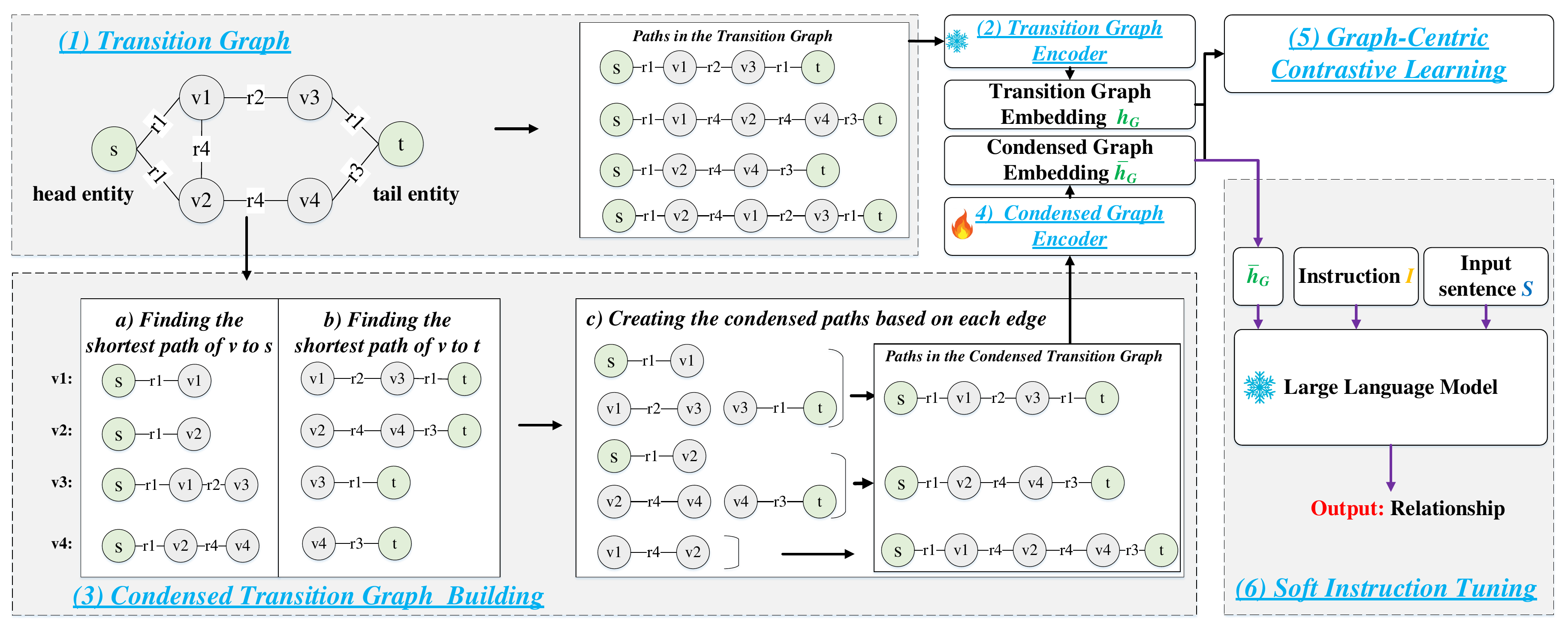}
        \vspace{-1mm}
	\caption{Overview of CTLP. During the training process, the language model parameters are frozen, and only the condensed graph encoder is trained.}
 \vspace{-1mm}
	\label{con:framework_overview}
\end{figure*}

\section{Proposed Approach: CTLP}
\subsection{Motivation and Approach Overview}
To achieve the link prediction between $s$ and $t$, the most commonly used way is to learn the embeddings of $s$ and $t$ with graph models (typically graph neural networks) and then calculate the similarity of their embeddings to predict their relation type $r$. Such a way, though efficient with $O(|E|)$, is limited in considering higher-order neighborhoods spanning $s$ and $t$ due to the inherently limited number of layers of graph models. However, higher-order consideration is critical, especially in zero-shot link prediction, as it provides a broader context with more chance to generalize across different domains. On the contrary, the thorough way that preserves all the information is to consider all the possible paths traversing from $s$ to $t$. However, the number of such paths is exponential to the number of edges $|E|$, which makes it prohibitive.

To address the above issue, we aim at a new framework that can not only preserve higher-hop relation information between $s$ and $t$ but also is efficient, with linear time complexity. More specifically, we propose a Condensed Transition Graph Enhanced LLM for the zero-shot Link Prediction (CTLP), which distills the information from all the paths between $s$ and $t$ to a \emph{condensed transition graph encoder} that can learn the information in the neighborhood of all the hops with $O(|E|)$ time complexity, as detailed in Section~\ref{sec:encoder} and theoretically analyzed in Section~\ref{sec:theories}, and illustrated in Figures~\ref{con:framework_overview}(3) and~\ref{con:framework_overview}(4). To let the condensed transition graph encoder better approximate the information from all the paths between $s$ and $t$, we design a contrastive learning strategy that minimizes their divergence, as detailed in Section~\ref{con:method_CL_part} and illustrated in Figures~\ref{con:framework_overview}(1),~\ref{con:framework_overview}(2), and ~\ref{con:framework_overview}(5). Finally, the learned embedding will be input into LLMs with the transferrer that is learned in Section~\ref{sec:instruction_tuning} and illustrated in Figure~\ref{con:framework_overview}(6).

%


\subsection{Condensed Transition Graph Building and Encoder}
\label{sec:encoder}

Here, the embedding of the information of each path between $s$ and $t$ through $(u,v)$ can be decomposed into: $(s,u)$ path embedding, $(u,v)$ edge embedding, and $(v,t)$ path embedding. To estimate the $(s,u)$ and $(v,t)$ path embedding, we have proposed different ways detailed in Section 4.2.
We demonstrate at the end of this section that such decomposition leads to a total time complexity of $O(|E|)$.

In this section, we propose our Condensed Transition Graph encoder that can calculate the embedding in linear time complexity (as demonstrated at the end of this section). More concretely, in the condensed transition graph encoder, we leverage the principle that all the paths between $s$ and $t$ can be split into each subset of paths between $s$ and $t$ going through each edge $(u,v)\in E$. Hence, the entirety of the embedding of all the paths between $s$ and $t$ is equivalent to the aggregation of the embeddings of paths pertaining to different edges $(u,v)\in E$, which is denoted as follows:
\begin{align}
	&\bar h_{\mathcal{G}}= \frac{1}{|E|} \sum\nolimits_{(u,v)\in E} \overline{h_{(u,v)}} \nonumber 
\end{align}
where $E$ is the edge set in $\mathcal{G}$. $\overline{h_{(u,v)}}$ denotes the embedding of all the paths between $s$ and $t$ through an edge $(u,v)$, and is calculated by the encoding of the composition of the three segments of this path, which is the key to achieving linear time complexity, as denoted in the following: 
\begin{align}
        &\overline{h_{(u,v)}}= CGE (h_{(s\rightarrow u,v\rightarrow t)}) \nonumber 
\end{align}
where $CGE$ is an encoder, such as the encoder part in a Transformer. $(s\rightarrow u,v\rightarrow t)$ is called a \emph{condensed path} via an edge $(u,v)$ and a \emph{condensed transition graph} $\mathcal G^*$ is defined as the composition of all the condensed paths via all the edges.   and $h_{(s\rightarrow u,v\rightarrow t)}$ is the concatenation of the embeddings  of the three segments of the path between $s$ and $t$ via $(u,v)$, namely:
\begin{align*}
&h_{(s\rightarrow u,v\rightarrow t)}= h_{(s\rightarrow u)} | h_{(u,v)} | h_{(v\rightarrow t)}
\end{align*}
where $h_{(s\rightarrow u)}$, $h_{(u,v)}$, and $h_{(v\rightarrow t)}$ are the paths' embedding from $s$ to $u$, edge embedding of $(u,v)$, and the paths' embedding from $v$ to $t$, respectively.  ``$|$'' is the concatenation function. This step is also illustrated in the section ``c)'' of Figure~\ref{con:framework_overview}(3). In the following, we elaborate our efficient methods to calculate the path embedding $h_{(s\rightarrow u)}$ and $h_{(v\rightarrow t)}$, as well as edge embedding $h_{(u,v)}$, respectively. 

The goal of embedding computation of the path $h_{(s\rightarrow u)}$ is to calculate the embedding that can enclose the correlation information between $s$ and $u$. An intuitive option is to aggregate all the paths' embedding between $s$ and $u$. However, a better trade-off between efficiency and information preservation can be a sampling or traversal through the most representative path(s) between $s$ and $t$. Hence, methods like the \emph{shortest path search} and \emph{fattest path search} between $s$ and $t$ provide a reasonable trade-off. 
Figure~\ref{con:framework_overview} provides an example of this shortest path collection, such as the shortest path of $v1$ to $s$ is $s-r1-v1$, its textual description $T((v1\rightarrow s))$   is ``the relationship between $s$ and $v1$ is $r1$''. The embedding of $T((v1\rightarrow s))$  is initialized by  a pre-trained language model, like LLaMA. Embedding computation of the path $h_{(v\rightarrow t)}$ can be calculated similarly, while the calculation of $h_{(u,v)}$ can be done by using  LLaMA to initialize its textual embedding for the statement "the relationship between $u$ and $v$ is $r$", $r$ is the relationship between $u$ and $v$.

\subsection{Complexity and Expressiveness of Condensed Transition Graph Encoder}
\label{sec:theories}

Here, we demonstrate the critical merits of our condensed transition graph encoder in terms of its coverage, expressiveness, and time complexity. 

\begin{lem}[Path Coverage]\label{lem: coverage}
    The condensed graph covers all paths from $s$ to $t$ when all relevant paths are within $k$ hops.
\end{lem}
\begin{proof}[Proof Sketch]
    By finding the shortest path from $s$ to every node $v$ in the transition graph $\mathcal{G}$, we ensure that we have a way to reach every node starting from $s$. This guarantees that any node that is part of a path from $s$ to $t$ in the original subgraph is reachable in the condensed graph $\mathcal{G}^*$ through one of these shortest paths. Similarly, by finding the shortest path from every node $v$ to $t$, we ensure there is a way to reach $t$ from any node $v$ in $\mathcal{G}$. This guarantees that for any part of a path from $s$ to $t$ that passes through any $v$, there exists a continuation to $t$ in the condensed graph $\mathcal{G}^*$. When all relevant paths are within $k$ hops, for any path from $s$ to $t$ in $\mathcal{G}$, it can be decomposed into segments where each segment starts and ends at nodes where the shortest paths from $s$ to $t$ were calculated. Since $\mathcal{G}^*$ contains all these shortest paths, the segments can be combined to form a path $\pi$ in $\mathcal{G}^*$  that corresponds to the original path.
\end{proof}

\begin{lem}[Expressiveness]
    The condensed transition graph $\mathcal{G}^*$ retains the same expressive power of the original transition graph $\mathcal{G}$ when all relevant paths are within $k$-hop in $\mathcal{G}$.
\end{lem}
\begin{proof}[Proof Sketch]
    First, $\mathcal{G}^*$ ensures the connectivity from $s$ to $v$ for every node $v$ in $\mathcal{G}^*$ as proved in Lemma \ref{lem: coverage}. Let $P_k(s, t)$ be the formula stating the existence of a path from $s$ to $t$ of length $k$ and let $E(s, t)$ be an edge connecting the nodes $s$ and $t$. $P_k(s, t)$ can be defined recursively as follows:
    \begin{align*}
        P_1(s,t)&\equiv E(s,t),\\
        P_k(s,t)&\equiv \exists_v P_{k-1}(s,v) \wedge E(v,t).
    \end{align*}
    Then, $\mathcal{G}^*$ and $\mathcal{G}$ can achieve the WL-equivalent (i.e., $\mathcal{G}^*$ and $\mathcal{G}$ have the same expressive power) for $k\ge 2$ according to Theorem 3.1 and 4.1 in \citep{graziani2023no}.
\end{proof}

\begin{lem}[Time complexity] The time complexity to approximate the embedding of all the paths between $s$ and $t$ by our Condensed Transition Graph Encoder is linear in the number of edges in $\mathcal{G}$.
\end{lem}
\begin{proof}[Proof Sketch] The calculation of embedding consists of two steps: 1) computing the path embedding of $(s\rightarrow u)$, edge embedding $(u,v)$, and path embedding $(v\rightarrow t)$, which can be done in $O(n)$, where $n$ is the number of nodes (if using popular algorithms like breath-first search for shortest path search around $s$ and $t$, respectively); and 2) aggregating the above embeddings for all the edges, which can be done in $O(m)$. Therefore, the total time complexity is $O(m)$ since $G$ is connected.
\end{proof}

\subsection{Graph-Centric Contrastive Learning}
\label{con:method_CL_part}
Here, we elaborate on how to train our condensed transition graph encoder proposed above.  Again, our objective is to let the embedding $\bar h_{\mathcal G}$ output by the condensed transition graph encoder preserve all the information carried by all the paths in $\mathcal G$. To achieve this, we propose a graph-centric contrastive learning (GCCL) method that minimizes the information gap between them. Specifically, for each $(s,t)$ pair,  GCCL aims to reduce the distance between the embedding $\bar h_{\mathcal G}$ from the condensed transition graph encoder and the embedding $h_{\mathcal G}$ extracted from all the paths in $\mathcal G$ (see the next paragraph about the computation of $h_{\mathcal G}$). To compose negative samples, we just use $(s,t)$ for one embedding while focusing on $(s',t')$ for the other, where $(s,t)\ne(s',t')$, by maximizing the distance between them. Note that one will tend to use small transition graph samples $\mathcal G$'s to train the condensed transition graph encoder efficiently.


Note that the embedding of all the paths in $\mathcal G$ is calculated as follows. Specifically,
all the paths between the $s$ and $t$ in the transition graph $\mathcal{G}$ constitute the initial paths set $\mathcal{P}$. For each textual path $T(\pi)$ in $\mathcal{P}$, its representation is learned from a pre-trained path encoder $f(\cdot)$,  which can be denoted as:
\begin{align*}
	&h_{T(\pi)}= f(T(\pi)), \quad h_{\mathcal{G}}= \frac{1}{|\mathcal{P}|} \sum\nolimits^{|\mathcal{P}|}_{\pi} h_{T(\pi)}
\end{align*}
where the representation of $\mathcal{G}$ (i.e., $h_{\mathcal{G}}$) is calculated by a mean function across all sample paths. 


\subsection{Soft Instruction Tuning}
\label{sec:instruction_tuning}
The embedding calculated by our condensed transition graph encoder can be used flexibly for downstream tasks. One important usage of it is a soft prompt into LLMs. To see this, here we introduce a soft instruction tuning method.
In particular, we put the aligned condensed graph embedding $\bar h_{\mathcal G}$, calculated through Graph-Centric Contrastive Learning, in the front of the original instruction $\mathcal{I}$ (\textit{The instruction to guide the LLM for the link prediction task}) and input sentence $\mathcal{S}$ (\textit{what is the relationship between $s$ and $t$ ?}). The $\bar h_{\mathcal G}$ serves as a soft prompt in the instruction.
In practice, we add several special tokens (e.g., [S]) to the beginning of the instruction, forming a sequence of length $l$, such as ([S] [S] [S]), where $l=3$, and then map the representation of $\bar h_{\mathcal G}$ into these token embeddings.
the concatenation of $\bar h_{\mathcal G}$, $\mathcal{\mathbf{I}}$ and $\mathcal{\mathbf{S}}$ are fed into LLM to predict the relationship between $s$ and $t$. $\mathcal{\mathbf{I}}$ and $\mathcal{\mathbf{S}}$  are the embedding of $\mathcal{I}$ and $\mathcal{S}$ separately. In the training progress, the LLM parameter is frozen. The overall objective for our proposed CTLP is the combination of the contrastive loss $CL$ (elaborated in Section~\ref{con:method_CL_part}) and cross-entropy loss $CE$:
\begin{align}
&L_{CE}=-\sum\nolimits_{i=1}^{k}log(\mathbf{y}_i|\mathbf{y}_{1:i-1},\bar h_{\mathcal G},\mathcal{I},\mathcal{S})  \nonumber\\ 
&L_{overall}=L_{CE} + L_{CL} \nonumber
\end{align} \label{eq:loss}
where, $\mathbf{y}_i$ is the generated token in the output relationship.

 \section{Experiments}
\subsection{Dataset}

In our experiments, we use three public KG benchmarks FB60K-NYT10~\citep{fu2019collaborative}, UMLS~\citep{kok2007statistical}, and NELL~\citep{das2017go} to evaluate the capability of the proposed CTLP.
FB60K-NYT10~\citep{fu2019collaborative} dataset includes the FB-60K knowledge graph and the NYT10 corpus. To further correct the data-building error, triples (head entity, relation, tail entity) are excluded if either the head entity or the tail entity is not found within the provided knowledge graph. UMLS~\citep{kok2007statistical} contains triples from the Unified Medical Language System~\citep{lindberg1993unified}, providing knowledge in the domain of healthcare and medicine. The NELL dataset, as presented by ~\citep{das2017go}, consists of individual graphs corresponding to each query relation, sourced from the web. The statistic information is shown in Table~\ref{con:LP data}.

\begin{table}[ht]
	\centering
	\caption{ Datasets Statistics, column 5 refers to the number of triples in the different set.}
	\renewcommand\arraystretch{1.3}
	\scalebox{0.7}{
	\begin{tabular} {ccccc}
		\hline 
		
		\hline	
		Dataset& $\mathcal{V}$ & $\mathcal{R}$ & \# Triples &   \# Train/ \# Dev/ \# Test\\ 
		\hline		
		FB60K-NYT10&9,840  &  56 &  13,837 & 12,104/\,--/\,1,733\\
		UMLS& 135   &46  &6,529   &5,216/\,652/\,661\\
            NELL& 9,093&12   & 12,108  &8,747/\,543/\,2,818\\
		\hline
	\end{tabular}}
		
	\label{con:LP data}
\end{table}

\subsection{Baselines and Evaluation Metrics}
In our experiments, the baselines include three commonly used KG embedding methods: TransE~\citep{bordes2013translating}, DisMult~\citep{yang2014embedding}  and ComplEx~\citep{trouillon2016complex}.
Obviously, these original models cannot predict the unseen relationship in the zero-shot setting. Therefore, based on these three methods, we propose three zero-shot baselines, \textbf{1) ZSL-TransE}, \textbf{2) ZSL-DistMult}, and \textbf{3) ZSL-ComplEx}. 
Specifically, the pre-trained BERT model~\citep{reimers2019sentence} is employed to compute embeddings for each entity and relation. Following this, the score for each candidate relation is determined using the scoring functions within TransE, DistMult, and ComplEx. \textbf{4) ZS-BERT}~\citep{chen2021zs},   it is designed specifically for zero-shot text classification, we adapt it to our zero-shot setting by calculating the similarity between the sentence \textit{what is the relationship between head entity and tail entity?} and the candidate relations.
Otherwise, We also compare the performance of \textsc{CTLP} with several strong baselines based on the LLM, including   \textbf{5) GPT-3.5/4}:  In GPT-3.5/4, we design hard prompts to guide the GPT models in predicting relations between the head and tail entities. \textbf{6) ZSL-InstructGLM}~\citep{ye2023natural}: we employ the same prompt construction method as InstructGLM~\citep{ye2023natural} to predict the relationship between two entities in zero-shot settings. We consider \textbf{7) LLAMA family} as the baselines: LLaMA2-(13, 70b)~\citep{touvron2023llama}.  
We report the standard Micro Precision, Recall,
and F1-score on the test set.

\begin{table*}[ht]
	\centering
 \caption{Results of various approaches for zero-shot link prediction on three open datasets. 
}
	\renewcommand\arraystretch{1.3}
\resizebox{0.8\textwidth}{!}{%
	\begin{tabular} {l|ccc|ccc|ccc}
		\toprule 
  
		\multicolumn{1}{c}  {}&\multicolumn{3}{c}  { FB60K-NYT10}& \multicolumn{3}{c}  {UMLS}& \multicolumn{3}{c}  {NELL}\\
		 Approach & Precision &  Recall & F1 &  Precision &  Recall & F1 &  Precision &  Recall & F1 \\
            \hline
            ZSL-TransE &7.27     & 7.27    & 7.27   &  3.78   &  3.78 &3.78    & 8.94 &  8.94 &  8.94  \\
            ZSL-ComplEx&  1.21   &  1.21  & 1.21  &  1.81   &  1.81  & 1.81    & 9.44 & 9.44  & 9.44   \\
            ZSL-DistMult &  7.09    &  7.09   &   7.09   & 4.08    & 4.08  & 4.08  & 7.94 &  7.94  &  7.94   \\
            ZS-BERT~\citep{chen2021zs} & 3.52    & 3.52    & 3.52   &  2.11   &  2.11  &   2.11  & 0.18 & 0.18  &0.18    \\
            ZSL-InstructGLM~\citep{ye2023natural} & 0.52&   0.52  &0.52    &   0.00    &   0.00 &0.00     &0.19  &  0.19  &  0.19    \\
            GPT-3.5  &  20.00   &   20.00    &   20.00    & 1.98   &    1.98  &   1.98   & 35.29  &  36.00  & 35.64   \\
            GPT-4  &    12.00 &   12.00  &  12.00&  6.00  &  6.00  &  6.00 & 39.00&39.00&   39.00 \\
            \hline
            LLaMA2 13b  & 36.18   &  36.18     &  36.18   &  9.83  & 9.83     &  9.83     &  38.18&   38.18  &    38.18   \\
            LLaMA2 13b + \textsc{CTLP}  &  \textbf{43.05}    & \textbf{43.05}       &\textbf{43.05}     &10.30 & 10.30    &  10.30    & 39.32  & 39.32     &    39.32  \\
            LLaMA2 70b  & 36.75  &  36.75      &   36.75   &   13.76 &  13.76 &      13.76 &  55.28 & 55.28  &55.28    \\
            LLaMA2 70b + \textsc{CTLP}  &  37.97   & 37.97      &37.97    &   \textbf{13.92}&   \textbf{13.92}  & \textbf{13.92}     & \textbf{56.78}  &  \textbf{56.78}    & \textbf{56.78}    \\
           \bottomrule
	\end{tabular}
 }
\vspace{-2mm}

\label{con:Model main performance}
\vspace{-5mm}
\end{table*}

\subsection{Implementation Detail}
To ensure the zero-shot setting, in the testing progress of dataset FB60K-NYT10, and UMLS,  the condensed graph encoder is trained in the knowledge graph NELL, and the LLM is frozen in the training progress of the condensed graph encoder. Following this, the trained condensed graph encoder is employed to encode the condensed transition graph in  FB60K-NYT10 and UMLS. The condensed graph encoder in the NELL dataset is trained using the FB60K-NYT10.  We set the hops $k=4$ in the transition graph of  FB60K-NYT10, UMLS, and NELL.   During the training process of the condensed graph encoder, the vector dimensions of the transition graph and the  LLM are the same.  In the dataset FB60K-NYT10, the length of the soft prompt token is set to $l=5$, while in the dataset UMLS and NELL, the length of the soft prompt token is set to $l=10$. The relation types of these three data sets do not overlap with each other.

\subsection{Main Results} 
Table~\ref{con:Model main performance} presents the experiment results of various approaches based on Precision, Recall and F1. 
We have the following observations: (1) our \textsc{CTLP} significantly outperforms all the strong baselines across all evaluation metrics. (2)  In the dataset FB60K-NYT10, we observe that  \textsc{CTLP} improve the original  LLaMA2 13B,  LLaMA2 70B  by 6.87\%, 1.2\%   respectively, in terms of F1. 
(3) \textsc{CTLP} does not significantly enhance the performance of LLM in UMLS. This can be attributed to the ongoing challenges faced by large language models (LLMs) in understanding path information within the biomedical domain, particularly due to the complexity of relations and entity names, such as abbreviations. (4) We observe that ZSL-InstructGLM gets zero performance in the UMLS, the reason is that ZSL-InstructGLM injects the hop3 information around the head and tail entity to the LLM, and this hard prompt design exceeds the maximum input length of LLM. (5) GPT-3.5/4 exhibits the lowest performance.  The main reason is that the reported results are in the zero-shot setting due to the unavailability of open resources.

\section{Analysis}
In order to further explore the effectiveness of our framework, we perform a series of analyses based on different characteristics of our model. 
First, we evaluate the effectiveness of our proposed CTLP with the models that feed all paths from the transition graph to the LLM.
The contribution of our model components can also be learned from ablated models. So we propose two model variants to help us validate the advantages of the contrative learning operation and condensed graph encoder.  Next, we explore the performance of our model with a different number of hops $k$ in the transition graph, and we also explore the model performance with different soft prompt token lengths $l$. Lastly, to further investigate the effectiveness of our model, we explore the model performance by using the limited dataset to train our graph encoder. 
Due to space constraints, please refer to Appendix~\ref{Impact_of_prompt_length} and Appendix~\ref{New_setting} for the parameter exploration and evaluation of model performance using the limited dataset.


\subsection{Comparison with All Paths Input}
\label{Comparison with All Paths Input}
\begin{table}[ht]
	\centering
 \caption{F1 Results of  \textsc{ CTLP}, \textsc{PathLLM-hop-k} and  \textsc{CPathLLM-hop-k}
 }
	\renewcommand\arraystretch{1.3}
\resizebox{0.45\textwidth}{!}{%
	\begin{tabular} {l|l|c|c|c}
		\toprule 
		& Approach & FB60K-NYT10 & UMLS& NELL \\
            \hline
            \multirow{5}*{LLaMA2 13b}& \textsc{PathLLM-hop-3}& 21.29  &    0.00  &  12.67     \\
                                    & \textsc{PathLLM-hop-4} & 4.22   & 0.00    & 3.10     \\
                                     & \textsc{CPathLLM-hop-3} & 21.29     &   0.00   &   12.67     \\
                                    & \textsc{CPathLLM-hop-4} & 4.33    & 0.00    &   3.21   \\
                                     & \textsc{CTLP }&\textbf{ 43.05}   &\textbf{ 10.30}    &\textbf{ 39.32  }  \\
             \hline
           \multirow{5}*{LLaMA2 70b}& \textsc{PathLLM-hop-3} & 30.92   &  0.00    &  18.29     \\
                                    & \textsc{PathLLM-hop-4} & 6.38    &  0.00     &6.39  \\
                                    & \textsc{CPathLLM-hop-3} &   30.92   &   0.00    &  18.29     \\
                                    & \textsc{CPathLLM-hop-4} & 6.26    &    0.00 & 6.42     \\
                                    &\textsc{CTLP} & \textbf{37.97}   &  \textbf{13.92}    &\textbf{56.78}  \\
           \bottomrule
	\end{tabular}
 }
\vspace{-2mm}
\label{con: results with PathLLM-hop-n}
\vspace{-1mm}
\end{table}

To further evaluate the effectiveness of our proposed CTLP, we introduce \textsc{PathLLM-hop-k} and  \textsc{CPathLLM-hop-k},  two LLMs that utilize path information to predict the relationship between the head and tail entities. In \textsc{PathLLM-hop-k}, the path information is derived from the transition graph, whereas in \textsc{CPathLLM-hop-k}, the path information is from the condensed transition graph.
More specifically, we first transfer all paths in the transition graph or condensed transition graph into their textual descriptions. For example, the path $s-r1-v1-r2-t$ is represented by \textit{the relationship between v1 and s is r1, the relationship between v1 and t is r2}. Secondly, these path descriptions are fed into the LLM to predict the relationship between the head entity and the tail entity in the format of the hard prompt.

Table~\ref{con: results with PathLLM-hop-n} shows the performance comparison.  Our model demonstrates superior performance on these datasets, outperforming the \textsc{PathLLM-hop-k} and \textsc{CPathLLM-hop-k}.
This is because when we provide all path information to the LLM, the prompt length exceeds the maximum input length of LLM, particularly evident in the UMLS dataset. As a result, \textsc{PathLLM-hop-k} achieves zero performance in UMLS.  Otherwise, despite our algorithm could reduce the number of paths, the condensed paths still surpass the maximum input length of the LLM. Consequently, \textsc{CPathLLM-hop-k} also exhibits a worse performance. 
Moreover, inputting excessively long textual descriptions to the LLM can also lead to a reduction in generation time. In contrast, our method employs a soft prompt strategy to alleviate this issue. Additionally, it utilizes contrastive learning to ensure comprehensive path information.  For the comparison of path numbers on the transition graph and condensed transition graph, please check Appendix~\ref{con:Path_Numbers_Comparison}.

\subsection{Ablation Experiments}
\label{Ablation Experiments}
We introduce two ablated models of CTLP: (1) {\bf CTLP-WCL} does not use the pre-computed transition graph information to guide the learning progress of condensed graph encoder; (2) {\bf CTLP-GCN} uses the traditional Graph Convolutional Network (GCN)~\citep{kipf2016semi} to encoder the condensed graph information. 
To ensure a fair comparison, we maintain consistency in the LLM, hop-k, and the length of prefix soft token across all datasets. We find that the performance of CTLP degrades as we remove important model components. Specifically,  both CTLP-WCL and CTLP-GCN perform poorly when compared to CTLP, indicating the importance of using the pre-computed transition graph information to guide the condensed graph encoder, thereby preventing the loss of valuable path information. Otherwise, the lower performance of CTLP-GCN  also indicates that GCN may face challenges in effectively learning comprehensive path information.

\begin{table}[ht]
	\centering
 \caption{\textsc{ CTLP} performance and its ablated model on the F1  score.
 }
	\renewcommand\arraystretch{1.3}
\resizebox{0.3\textwidth}{!}{%
	\begin{tabular} {l|c|c|c}
		\toprule 
		Approach & FB60K-NYT10 & UMLS& NELL \\
            \hline
            \textsc{CTLP-WCL} & 39.12  & 13.46   & 55.93    \\
             \textsc{CTLP-GCN}&   35.49 &   13.61  &  55.78    \\
              \textsc{CTLP }&\textbf{ 43.05}   &\textbf{ 13.92}    &\textbf{ 56.78}  \\
           \bottomrule
	\end{tabular}
 }
\vspace{-2mm}
\label{con: ablated model}
\vspace{-5mm}
\end{table}


\section{Conclusion}

In this paper, we introduce CTLP, a novel ZSL framework for link prediction. Our approach focuses on leveraging all sample paths between the head and tail entities to predict their relationship. To achieve this, we develop a condensed transition graph construction method and employ contrastive learning to balance time efficiency and comprehensiveness when encoding these paths. Subsequently, the learned condensed transition graph is used as the soft prompt to feed into the LLM.
Experimental results show that our framework achieves consistent improvements over various baselines in three open datasets.

\section{ACKNOWLEDGMENT}
This work was supported by the National Institutes of
Health’s National Center for Complementary and Integrative
Health under grant number R01AT009457, National Institute on Aging under grant number R01AG078154 and
National Cancer Institute under grant number
R01CA287413. The content is solely the responsibility of the
authors and does not represent the official views of the
National Institutes of Health.

This work was supported by the National Science Foundation (NSF) Grant No. 2403312, No. 1841520, No. 2007716, No. 2007976, No. 1942594, No. 1907805, and Amazon Research Award.



\bibliography{IEEEfull}
\bibliographystyle{IEEEtran}
\section{Appendix}
\subsection{Path Numbers Comparison}
\label{con:Path_Numbers_Comparison}

Table~\ref{con:comparison_of_path_numbers} presents a comparison of path numbers between the transition graph and the condensed transition graph. The results demonstrate that our method effectively reduces the path numbers on the transition graph.


\begin{table}[ht]
	\centering
 \caption{ A comparison of path numbers within hop 4  in transition graph and condensed transition graph. Each value refers to the average path numbers on the test set across the relevant dataset.
 }
	\renewcommand\arraystretch{1.3}
\resizebox{0.4\textwidth}{!}{%
	\begin{tabular} {l|c|c|c}
		\toprule 
		 & FB60K-NYT10 & UMLS& NELL \\
            \hline
            path numbers in transition graph & 4,689  &   12,617  &   1,220 \\
             path numbers in condensed transition graph & 1,748  &  558   & 524   \\
             \hline
	\end{tabular}
 }
\vspace{+2mm}
\label{con:comparison_of_path_numbers}
\vspace{-5mm}
\end{table}

\subsection{Impact of Hop-$k$ and the Length $l$ of Soft Prompt Tokens 
 }\label{Impact_of_prompt_length}



\begin{figure}[htbp]
    \centering
    \subfigure[]{
        \includegraphics[width=2in]{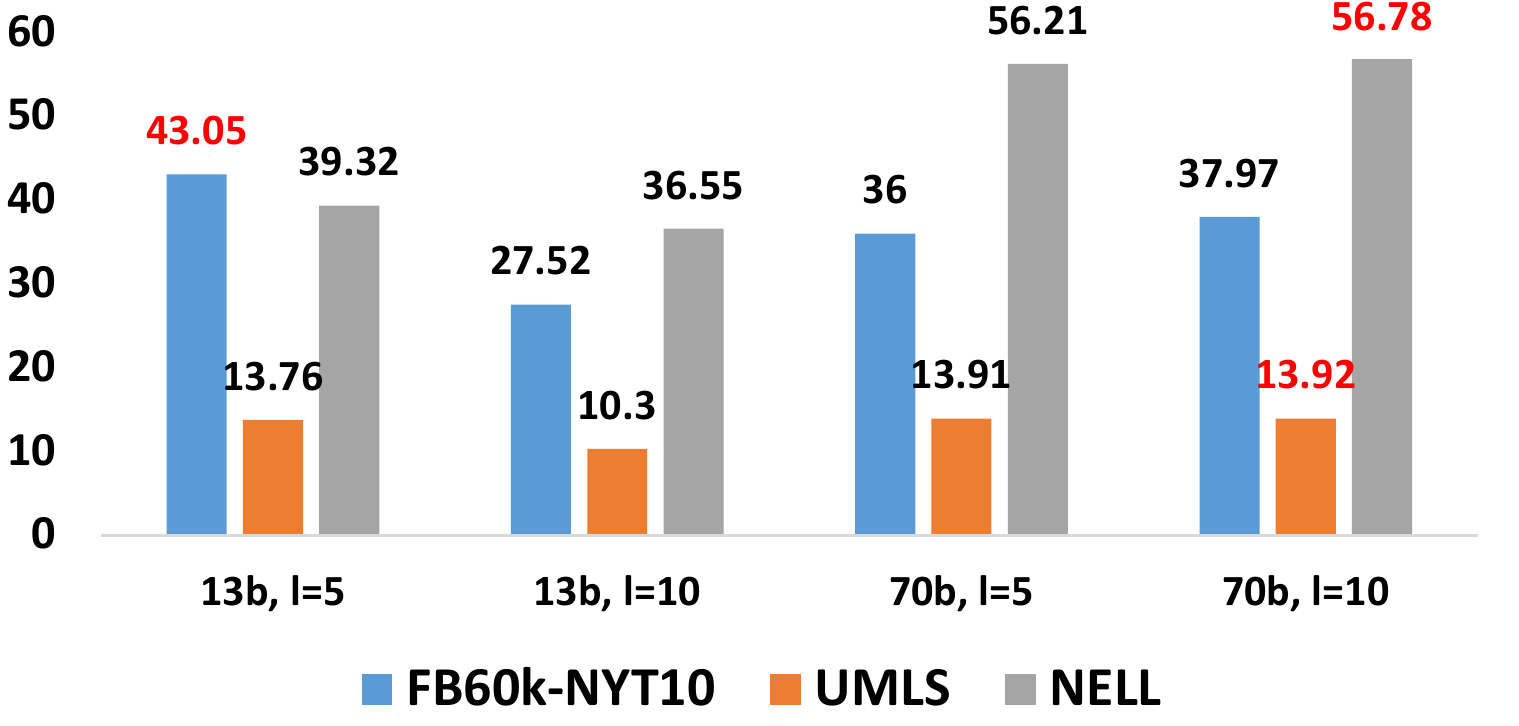}
    }
    \subfigure[]{
	\includegraphics[width=2in]{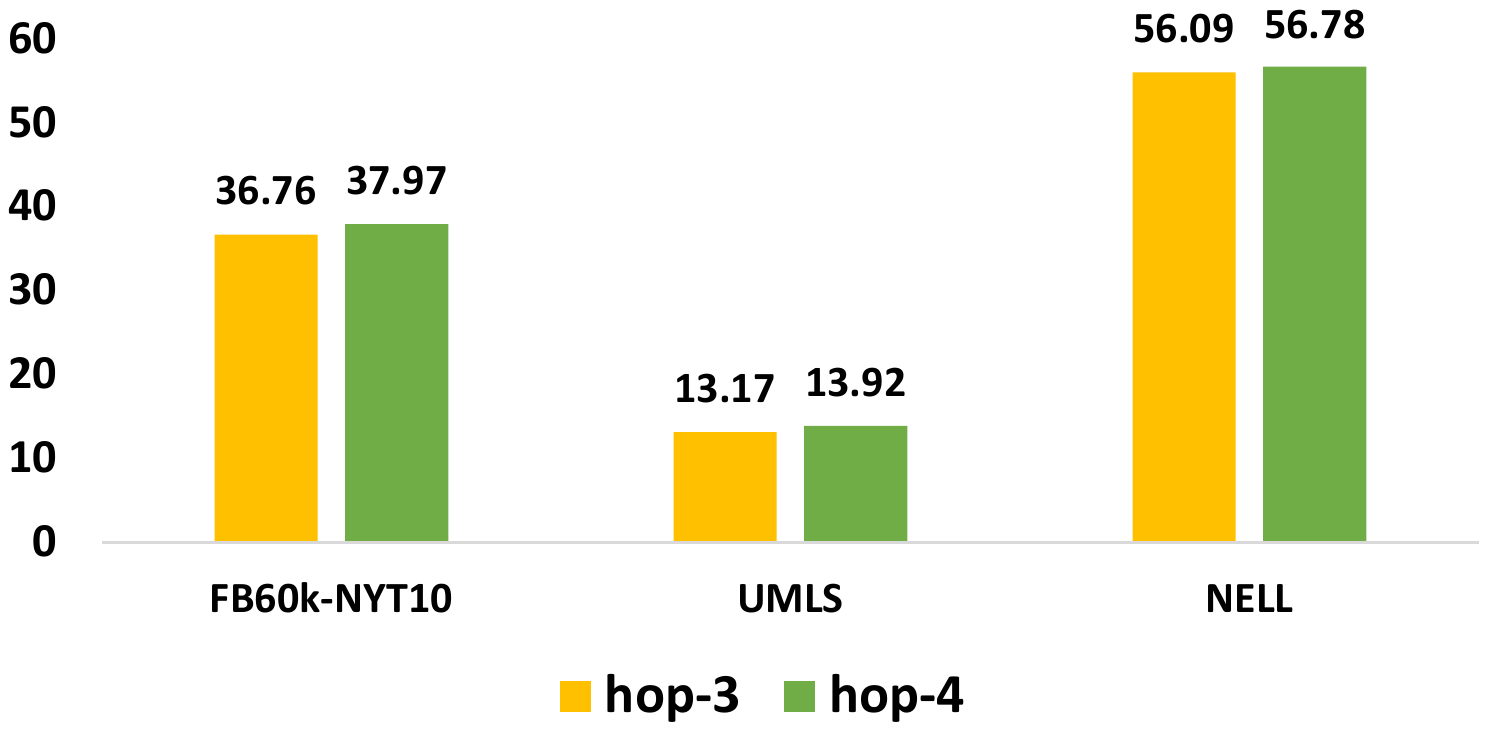}
    }
    \quad    
    \caption{(a): \textsc{CTLP} performance with different  prompt token length ($l$) on the F1 score. We fix the hop-$k=4$. 13b denotes the LLaMA2 13b, 70b denotes the LLaMA2 70b. (b):\textsc{CTLP} performance with different  hop-$k$ on the F1 score. We fix the LLM as LLaMA 70b, and the $l=10$. }
    \label{con:parameters}
\end{figure}

In our proposed model, there is one parameter controlling the size of the transition graph, it is hop-$k$. And the other important parameter is the length $l$ of the soft prompt tokens in CTLP, we treat these two parameters with the same importance.  From Figure~\ref{con:parameters}(a), we observed that the model performs better under hop-4 than under hop-3. This indicates that richer path information is beneficial for relationship prediction.  From Figure~\ref{con:parameters}(b), For the UMLS and NELL datasets, our model demonstrates superior performance when $l=10$, whereas its effectiveness is more pronounced on the FB60k-NYT10 dataset when $l=5$.


\subsection{Train the Condensed Graph Encoder using a Limited Dataset}
\label{New_setting}


In our study, we initially trained the condensed graph encoder on the knowledge graph NELL. Subsequently, we employ the trained condensed graph encoder to encode the condensed transition graphs in the UMLS and FB60K-NYT10 datasets. This approach ensures compatibility with the zero-shot setting.
In this section, we mainly focus on investigating the effectiveness of our model when utilizing a constrained dataset for training the condensed graph encoder.  For example, when training the condensed graph encoder on the knowledge graph NELL, we mask  $\%30/50$ of the relation types along with their associated entity pairs in the training dataset. The selection is based on the relation set identified in the test dataset.  They are denoted as \textsc{CTLP  mask 30\%} and   \textsc{CTLP  mask 50\%}.

The results are presented in Table~\ref{con:mask dataset}. It is noteworthy that our model consistently outperforms the original LLM even when the mask operation is applied. The model performance even beyond the \textsc{CTLP}, indicating that our model achieves superior performance with a limited dataset.

\begin{table}[ht]
	\centering
 \caption{\textsc{ CTLP} performance with different settings. For the FB60K-NYT10 dataset, these models are built upon LLaMA2 13b, whereas for the UMLS and NELL datasets, these models are built upon LLaMA2 70b.
 }
	\renewcommand\arraystretch{1.3}
\resizebox{0.4\textwidth}{!}{%
	\begin{tabular} {l|c|c|c}
		\toprule 
		Approach & FB60K-NYT10 & UMLS& NELL \\
            \hline
              \textsc{LLaMA2 13b} & 36.18  & --  & -- \\
                \textsc{LLaMA2 70b} &  -- &   13.76& 55.28  \\
              \textsc{CTLP  mask 50\%}& 36.98  & 13.92   &   \textbf{ 56.85}\\
                \textsc{CTLPmask 30\%}  &  42.35 &    \textbf{14.22} &  56.42 \\
                \textsc{CTLP}&\textbf{ 43.05}   & 13.92    &56.78  \\
           \bottomrule
	\end{tabular}
 }
\vspace{+2mm}
\label{con:mask dataset}
\vspace{-5mm}
\end{table}

\end{document}